\documentclass[11pt]{article}
\usepackage[margin=1in]{geometry}
\usepackage{times}
\newcommand{\firstpage}[1]{}
\newcommand{\correspondance}[1]{}
\newcommand{\keymark}[1]{\textbf{#1}}

\usepackage[utf8]{inputenc}
\usepackage[T1]{fontenc}
\usepackage{amsmath,amssymb,amsthm}
\usepackage{graphicx}
\usepackage{booktabs}
\usepackage{hyperref}
\usepackage{xcolor}
\usepackage{algorithm}
\usepackage{algorithmic}
\usepackage{natbib}
\usepackage{subcaption}
\newtheorem{theorem}{Theorem}
\newtheorem{proposition}[theorem]{Proposition}
\newtheorem{corollary}[theorem]{Corollary}
\newtheorem{definition}[theorem]{Definition}

\newcommand{\lrContent}{\alpha}       
\newcommand{\lrRecord}{\gamma}        
\newcommand{\pInject}{\rho}           
\newcommand{\thresh}{\varphi}         
\newcommand{\turnover}{\lambda}       

\begin{document}

\title{Persistent Memory Through Triple-Loop Consolidation\\
  Under Stochastic Unit Turnover}

\author{%
  Jianwei Lou$^{1}$\\[4pt]
  {\small $^{1}$RailMind Systems, Neuss, Germany}\\[2pt]
  {\small \texttt{j.lou@railmind.eu}}%
}
\date{}

\maketitle

\begin{abstract}
Dissipative cognitive architectures maintain computation through continuous
energy expenditure, where units that exhaust their energy are stochastically
replaced with fresh random state. This creates a fundamental challenge: how
can persistent, context-specific memory survive when all learnable state is
periodically destroyed? Existing memory mechanisms---including elastic weight
consolidation, synaptic intelligence, and surprise-driven gating---rely on
gradient computation and are inapplicable to dissipative systems that do not perform it.

We introduce \emph{Deep Memory} (DM), a backpropagation-free persistent memory
mechanism operating through a triple-loop consolidation cycle:
(1)~\emph{recording} of expert-specific content centroids during active
computation, (2)~\emph{seeding} of replaced units with stored
representations, and (3)~\emph{stabilization} through continuous re-entry
that counteracts dissipative drift. We demonstrate that discrete expert
routing via Mixture-of-Experts (MoE) gating is required, in the regimes tested,
for DM functionality, preventing the centroid convergence that would render
stored memories identical.

Across $972$ simulation runs spanning thirteen experimental
blocks, we establish: (i)~removing stable context--expert binding removes
structural specialization ($\text{MI} = 1.10$ vs.\ $0.001$
under binding destruction; $n = 91$); (ii)~DM achieves persistent
representation quality $R = 0.984$ compared to $0.385$ without memory
($n = 16$); (iii)~continuous seeding reconstructs representations after
interference ($R_{\text{recon}} = 0.978$; one-shot seeding fails;
$n = 30$); (iv)~the mechanism operates within a well-characterized
$(K, \pInject)$ envelope with identifiable degradation boundaries ($n = 350$);
(v)~a factorial ablation isolates recording $\times$ seeding as the minimal
critical dyad ($n = 40$); and (vi)~six confirmatory experiments---including
single-factor causal ablation, mediation analysis, and comparison with
associative and reservoir baselines (Hopfield, ESN) under matched turnover---validate
what the mechanism requires and how it compares with those baselines ($n = 370$).
Results (ii)--(v) rest on pilot-scale seed counts (4--5 seeds per condition);
at these counts no paired non-parametric test can reach $p < 0.05$ whatever
the effect size, so we report them as estimates of effect size and direction
rather than as confirmatory tests.
Within these limits, DM is a falsifiable, bounded mechanism for
persistent memory in backpropagation-free cognitive systems. Its relation to
hippocampal consolidation is a functional analogy only: DM shares the premise
of a store held apart from the computing substrate, but assigns the fast and
slow roles in the opposite direction to complementary learning systems theory.

\tiny
\keymark{Keywords:} persistent memory, dissipative systems, mixture-of-experts,
backpropagation-free learning, cognitive architecture, memory consolidation,
stochastic unit turnover, autopoiesis
\end{abstract}

\section{Introduction}
\label{sec:intro}

Biological and artificial cognitive systems that maintain function through
continuous energy expenditure face a fundamental memory challenge. In
gradient-based architectures, parameters accumulate information over training
and persist indefinitely. In \emph{dissipative} architectures---systems where
computational units consume energy, compete for resources, and are
stochastically replaced when depleted---all learnable state is periodically
destroyed \citep{prigogine1984order, maturana1980autopoiesis}. The expected
lifetime of any individual unit's state is on the order of $1/\turnover$ computation
steps, where $\turnover$ is the replacement rate. Under typical operating
conditions, this implies that no single unit survives long enough to accumulate
cross-episode information.

This creates a gap in the existing landscape of memory mechanisms.
Catastrophic forgetting \citep{french1999catastrophic} is a well-studied
problem in gradient-based networks; solutions such as elastic weight
consolidation \citep{kirkpatrick2017ewc}, synaptic intelligence
\citep{zenke2017si}, memory-aware synapses \citep{aljundi2018mas}, and
progressive networks \citep{rusu2016progressive} protect important
parameters, but all require gradient access to identify which parameters
matter. The Titans architecture \citep{titans2024} introduced
surprise-driven memory gating, but its memory updates rely on gradient
descent. Backpropagation-free approaches---including classical Hopfield networks
\citep{hopfield1982neural}, reservoir computing
\citep{jaeger2001esn, lukosevicius2009reservoir}, and attractor-based
working memory models \citep{wang2001synaptic, compte2000synaptic}---maintain
representations without backpropagation, but assume fixed network
topology without stochastic unit replacement. Complementary learning
systems theory \citep{mcclelland1995cls, kumaran2016hippocampus}
offers a functional reference point---fast hippocampal encoding
complemented by slow neocortical consolidation
\citep{frankland2005consolidation, rasch2013sleep}---but we are not aware of
an implementation of its separate-store premise in backpropagation-free
dissipative systems. We borrow only that premise; as Section~\ref{sec:dm}
explains, the fast and slow roles are reversed in our architecture. The core
question remains: \emph{can a dissipative system with stochastic unit
turnover maintain persistent, context-specific memory without gradient
computation?}

We answer this question affirmatively by introducing \emph{Deep Memory}
(DM), a triple-loop consolidation mechanism that maintains persistent
representations through recording, seeding, and stabilization
(Figure~\ref{fig:architecture}). Our contributions are:

\begin{enumerate}
  \item \textbf{Deep Memory (DM)}: A triple-loop consolidation mechanism
    (recording $\to$ seeding $\to$ stabilization) that maintains persistent
    representations without gradient computation, achieving $R = 0.984$
    under standard conditions (16 runs, 4 seeds; pilot scale, see
    Section~\ref{sec:protocol}).
  \item \textbf{Discrete routing is required in the tested regimes}: We show
    (Proposition~\ref{thm:t3}) and empirically confirm (91 runs, 7 seeds per condition) that
    removing discrete expert routing removes DM's ability to produce
    context-specific memory, with mutual information dropping from $1.10$ to
    $0.001$ when binding is destroyed.
  \item \textbf{Operating envelope}: We characterize the $(K, \pInject)$ parameter
    space across 350 runs (5 seeds per operating point), identifying pass,
    degraded, and failure regimes and the boundaries between them.
  \item \textbf{Scheduling invariance}: DM quality is invariant across five
    qualitatively different context-scheduling patterns (75 runs, 5 seeds).
  \item \textbf{Minimal mechanism via ablation}: A $2^3$ factorial ablation
    (40 runs, 5 seeds) isolates recording $\times$ seeding as the smallest combination
    that reproduces the full effect.
\end{enumerate}
Except for the single-factor ablation (14 seeds per condition;
Section~\ref{sec:res-p1}), these results are pilot-scale: with 4--5 seeds per
condition they estimate the size and direction of each effect but cannot
reach conventional significance in a paired non-parametric test, whatever the
effect size (Section~\ref{sec:protocol}).

\section{Materials and Methods}
\label{sec:methods}

\subsection{Dissipative Cognitive Grid}
\label{sec:grid}

The system consists of $N$ computational units with structured local
connectivity. Each unit $i$ maintains a content vector $z_i \in
\mathbb{R}^D$ encoding its current representational state and a scalar energy
$E_i$ governing its metabolic viability. Computation proceeds in discrete
cycles, each comprising the following steps:
The population size $N$ and the content dimension $D$ are modelling choices
rather than claims about the granularity of cognition. A finite, discrete
population cannot represent a continuous-valued substrate exactly, and we make no
argument that it should. $N$ is chosen large enough that the population
statistics on which the mechanism depends --- per-expert means and their
dispersion --- are stable; the results below are properties of that regime,
not limits as $N \to \infty$, and no claim is made that increasing $N$
converges to a continuous system.

\begin{enumerate}
  \item \textbf{Energy allocation}: External input $x_t \in \mathbb{R}^D$
    provides energy to units based on local activity.
  \item \textbf{Activation}: Units with sufficient energy ($E_i > E_{\min}$)
    and suprathreshold input produce activation $a_i$, modulated by a
    per-unit homeostatic threshold $\thresh_i$.
  \item \textbf{Content update}: Active units update their content via an
    exponential moving average toward the neighborhood-weighted input:
    $z_i \leftarrow (1 - \lrContent)\,z_i + \lrContent\,\bar{x}_i$.
  \item \textbf{Metabolic cost}: Each activation incurs an energy cost
    $c(a_i)$, depleting $E_i$.
  \item \textbf{Stochastic replacement}: Units with $E_i < E_{\min}$ are
    replaced with fresh random content and energy, implementing the
    dissipative turnover that defines the architecture.
\end{enumerate}

The homeostatic threshold $\thresh_i$ adapts to maintain a target firing rate,
ensuring stable population activity despite energy fluctuations. This
architecture is \emph{backpropagation-free}: no system-level loss function is defined, no
backpropagation occurs, and all learning arises from local Hebbian-type rules
and local adaptation dynamics.

We use \emph{backpropagation-free} rather than \emph{gradient-free}
advisedly. A local rule that moves a stored quantity toward its input---the
moving average of Eq.~\eqref{eq:recording} below, for instance---can always
be rewritten as stochastic gradient descent on a local quadratic objective,
and we do not claim otherwise. That equivalence holds for essentially any
first-order online rule, however, and therefore does not separate one
architecture from another. What does separate them is the \emph{scope of
credit assignment}: this system defines no objective over its own output,
propagates no error signal backwards across layers, and performs no weight
transport. Every update uses only quantities locally available at the unit
or expert being updated.

\subsection{Discrete Expert Routing}
\label{sec:moe}

To enable context-specific computation, the $N$ units are partitioned into
$K$ expert groups of $N/K$ units each. At each computation step, a single
expert group $k^*$ is selected based on input similarity:
\begin{equation}
  k^* = \arg\max_k \operatorname{sim}(x_t, \mu_k)
\end{equation}
where $\mu_k$ is the running centroid of inputs routed to expert $k$. Only
units in the selected group undergo activation and content update; all other
groups are inhibited. Each expert maintains an independent homeostatic
threshold $\thresh_k$. This implements \emph{hard} (discrete) routing with no
soft mixing or gradient-based gating---a key distinction from standard MoE
architectures \citep{jacobs1991adaptive, shazeer2017moe, fedus2022switch}.

Three properties of this rule are used later and are worth making explicit.
First, $\mu_k$ is maintained online from the inputs actually routed to
expert $k$; the router is an estimator running on the same stream as the
rest of the system and is given no target signal. Second, the
\emph{partition} $\{G_k\}$ --- which unit belongs to which expert --- is
fixed at initialisation and is never reassigned. Stochastic replacement
overwrites a unit's content and energy (the replacement step of
Algorithm~\ref{alg:cycle})
but leaves its group membership untouched, so what turnover perturbs is the
content an expert holds, not the identity of the units holding it. This
distinction is what allows the drift analysis of
Section~\ref{sec:dm-theory} to treat the partition as fixed while the
contents move. Third, because selection is an argmax rather than a
differentiable mixture, no gradient passes through the routing decision;
the routing layer therefore sits inside the backpropagation-free regime of
Section~\ref{sec:grid} rather than being an exception to it.

\begin{algorithm}[t]
\caption{Computation Cycle with MoE and Deep Memory}
\label{alg:cycle}
\begin{algorithmic}[1]
\REQUIRE Input $x_t \in \mathbb{R}^D$, expert count $K$, grid state, DM state $\{m_k\}$
\STATE $k^* \leftarrow \arg\max_k \operatorname{sim}(x_t, \mu_k)$
  \COMMENT{Select expert}
\STATE Activate units in group $k^*$; inhibit others
\FOR{each active unit $i$ in group $k^*$}
  \STATE $a_i \leftarrow f(E_i, \thresh_{k^*}, \text{neighbors}(i))$
  \STATE $z_i \leftarrow (1 - \lrContent)\,z_i + \lrContent\,\bar{x}_i$
    \COMMENT{Content update}
  \STATE $E_i \leftarrow E_i - c(a_i)$
    \COMMENT{Metabolic cost}
\ENDFOR
\STATE $m_{k^*} \leftarrow (1 - \lrRecord)\,m_{k^*} + \lrRecord\,\bar{z}_{k^*}$
  \COMMENT{DM recording}
\FOR{each unit $i$ where $E_i < E_{\min}$}
  \STATE Replace $z_i, E_i$ with random initialization
  \IF{unit $i \in$ group $k$ \textbf{and} $\text{Bernoulli}(\pInject) = 1$}
    \STATE $z_i \leftarrow m_k$
      \COMMENT{DM seeding}
  \ENDIF
\ENDFOR
\end{algorithmic}
\end{algorithm}

\subsection{Deep Memory: Triple-Loop Consolidation}
\label{sec:dm}

Deep Memory operates through three concurrent loops that together maintain
persistent representations despite stochastic unit turnover
(Algorithm~\ref{alg:cycle}, lines 8--13):

\textbf{Loop 1: Recording.} For each expert group $k$, DM maintains a
running centroid $m_k \in \mathbb{R}^D$ via exponential moving average of
the content vectors of active units:
\begin{equation}
  m_k \leftarrow (1 - \lrRecord)\,m_k + \lrRecord\,\bar{z}_k
  \label{eq:recording}
\end{equation}
where $\bar{z}_k = \frac{1}{|F_k|}\sum_{i \in F_k} z_i$ is the mean content
of firing units in group $k$ and $\lrRecord$ is the recording rate.

\textbf{Loop 2: Seeding.} When a unit $i$ in group $k$ is stochastically
replaced, its content is initialized to the stored centroid $m_k$ with
injection probability $\pInject$:
\begin{equation}
  z_i^{\text{new}} =
  \begin{cases}
    m_k & \text{with probability } \rho \\
    z_{\text{random}} & \text{with probability } 1 - \rho
  \end{cases}
  \label{eq:seeding}
\end{equation}

\textbf{Loop 3: Stabilization.} Because unit turnover is continuous and
seeding occurs at every replacement event, DM creates a re-entry loop: the
centroid $m_k$ is continually refreshed by active units (Loop~1) and
continually injected into replacement units (Loop~2). This closed loop
counteracts the dissipative drift that would otherwise cause content vectors
to diverge from their expert-specific attractors.

The triple-loop architecture distinguishes DM from replay-based memory
\citep{carr2011hippocampal, lin1992reinforcement}: DM does not store and
replay individual experiences, but rather maintains a running
population-level summary \citep{pouget2000information} that is continuously
re-injected into the active computational substrate.

\paragraph{Scope of the term.} Deep Memory is an \emph{explicit
architectural component}, not an emergent property. The name reflects the
layered organisation of the architecture---DM occupies a dedicated
consolidation layer above the routing substrate---and carries no claim that
persistence arises spontaneously from the grid dynamics.
Equation~\eqref{eq:recording} states plainly what is stored and where: a
per-expert centroid $m_k$ maintained outside the population of computational
units, so that it survives the replacement of any individual unit. What this
paper claims is that such an explicit, structurally separate store is
sufficient to maintain context-specific representations under continuous
stochastic turnover, and that recording and seeding are jointly required for
it to do so (Section~\ref{sec:res-ablation}). It does not claim that the
store is self-organising, nor that persistence is an emergent consequence of
the dissipative dynamics alone.

A store held apart from the computing substrate is the design premise of
complementary learning systems \citep{mcclelland1995cls,
kumaran2016hippocampus}, in which the hippocampus is anatomically distinct
from the neocortical networks whose representations it helps stabilise.
\emph{The direction of the fast--slow division is, however, inverted here},
and we state this explicitly rather than adopt the analogy wholesale. In CLS
the hippocampus encodes rapidly while the neocortex integrates slowly; in
DM the stored centroid $m_k$ is itself the \emph{slow} quantity---an
exponential moving average with rate $\lrRecord$
(Eq.~\ref{eq:recording})---while the active units it seeds are the fast,
continually replaced component. The inversion follows from what is absent:
this architecture has no gradient-based cortical learner, so slow
integration has no backward pass to perform it and is carried by the moving
average instead. Section~\ref{sec:hippocampal} returns to the limits this
places on the analogy.

\begin{figure}[t]
  \centering
  \includegraphics[width=0.85\linewidth]{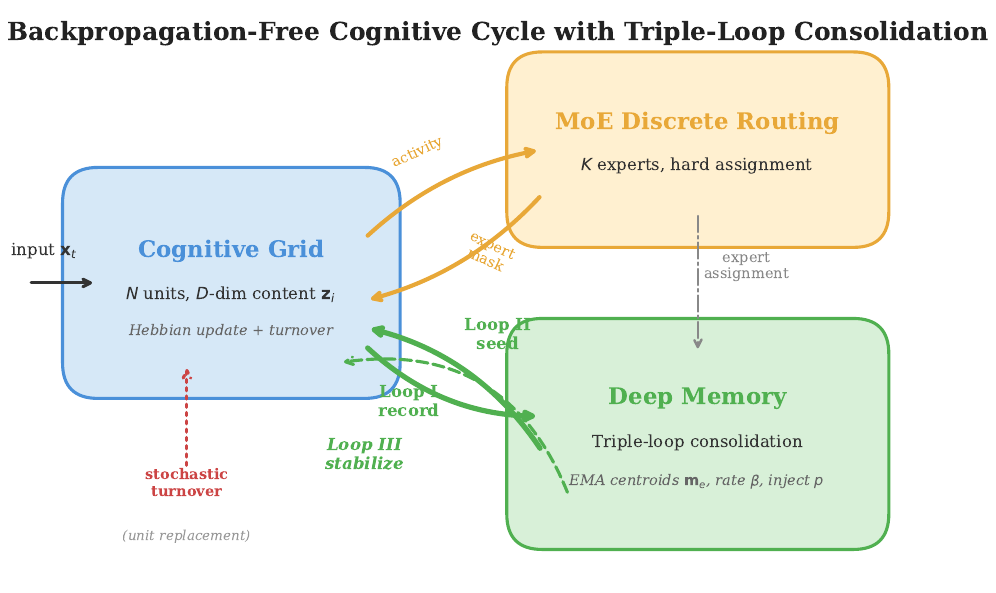}
  \caption{Schematic of the backpropagation-free cognitive cycle with triple-loop
    consolidation. Content activations in the grid are routed to one of $K$
    experts; Deep Memory records expert-specific centroids (Loop~I),
    reseeds replaced units (Loop~II), and continuously stabilizes
    representations against dissipative drift (Loop~III). Stochastic
    turnover periodically replaces units, creating the persistence
    challenge that the triple-loop addresses.}
  \label{fig:architecture}
\end{figure}

\subsection{Experimental Protocol}
\label{sec:protocol}

All experiments use synthetic multi-context tasks where $K$ distinct input
distributions, each with a unique centroid in $\mathbb{R}^D$, are presented
in alternating blocks or stochastic schedules. We evaluate the following
metrics:

\begin{itemize}
  \item \textbf{Representation quality} $R$: Cosine similarity between each
    expert group's mean content vector and the ground-truth centroid of its
    assigned context, averaged across experts. $R = 1$ indicates perfect
    alignment; $R = 0$ indicates chance.
  \item \textbf{Firing selectivity} $f_{\text{sel}}$: Fraction of units
    that fire exclusively in their assigned context (versus firing across
    multiple contexts).
  \item \textbf{Mutual information} $\text{MI}(k, c)$: Information shared
    between expert assignment $k$ and context label $c$. Maximum $\ln K$
    for perfect binding.
  \item \textbf{Silhouette score}: Mean silhouette across context-labeled
    content vectors, measuring structural separation.
\end{itemize}

Statistical testing uses multiple random seeds per condition with
median reporting and bootstrap 95\% confidence intervals. The results
reported here use 4--7 seeds per condition as pilot-scale estimation
of effect sizes and directions. For confirmatory claims, we target
$\geq 14$--20 seeds per condition (powered for medium effect sizes,
Cohen's $d \geq 0.5$, at $\alpha = 0.05$, $1 - \beta = 0.80$).
Multi-factor stress tests (binding disruption via per-cycle random
permutation of expert assignments) are used to establish necessary
conditions; single-factor ablations are identified as a priority for
confirmatory causal identification. For every figure and table,
Supplementary Table~\ref{tab:s-repro} lists the experimental block, its
conditions, the number of runs, the exact random seeds, and every parameter
value that differs from the simulator defaults; the defaults themselves are
given in Supplementary Data Sheet~1 (see the Code Availability Statement).

Two consequences of the pilot-scale seed count should be read alongside every
interval we report. First, with $n=5$ seeds a bootstrap interval for a median
can only take values in the sample itself, so the intervals below describe the
spread of the observed seeds rather than a smooth estimate of sampling
uncertainty. Second, and more sharply, the smallest attainable $p$-value of a
two-sided Wilcoxon signed-rank test at $n=5$ is $2^{-4} = 0.0625$: no
paired non-parametric test at this seed count can reach $p<0.05$, whatever the
effect size. Where we report paired differences we therefore give the interval
and the effect size, and treat the accompanying parametric $p$-values as
descriptive rather than confirmatory. This is the concrete form of the
pilot-scale caveat stated above, and it is why single-factor confirmatory
ablation at $\geq 14$--20 seeds remains listed as required future work.

\section{Theoretical Framework}
\label{sec:theory}

\subsection{Centroid Collapse Under Uniform Activation}
\label{sec:t3}

We first establish why persistent memory is non-trivial in dissipative
architectures. The content update rule drives content vectors toward
the local input average. Under what we term \emph{uniform activation},
this creates an inevitable contraction toward a single grand centroid.
We formalize the conditions under which this collapse occurs:

\begin{definition}[Uniform Activation (UA)]
\label{def:ua}
A system satisfies Uniform Activation if:
\textup{(UA-1)} the firing support is context-invariant: for all contexts
$k$, each unit $i$ is updated with the same probability;
\textup{(UA-2)} mixing weights are equal: context sampling follows
$\pi_k = 1/K$ for all $k$; and
\textup{(UA-3)} updates are unbiased: the expected input to unit $i$ under
context $k$ satisfies $\mathbb{E}[\bar{x}_i \mid k] = c_k$.
\end{definition}

\begin{proposition}[Grand-Centroid Collapse]
\label{thm:t3}
Consider $N$ units with content vectors $\{z_i\}_{i=1}^N$ updated via the
exponential moving average rule $z_i \leftarrow (1 - \lrContent)\,z_i +
\lrContent\,\bar{x}_i$ across $K$ contexts with centroids $\{c_k\}_{k=1}^K$.
Under Uniform Activation (Definition~\ref{def:ua}),
$\lim_{t \to \infty} \operatorname{Var}(\{z_i\}) = 0$
and all content vectors converge to the grand centroid
$\bar{c} = \frac{1}{K}\sum_k c_k$.
\end{proposition}

This result is a direct consequence of classical stochastic approximation
dynamics under equal mixing weights \citep{robbins1951stochastic}; we
include it as a diagnostic baseline that formalizes the collapse failure
mode induced by uniform activation.

\begin{proof}[Proof sketch]
Under UA, each unit $i$ receives updates from all $K$
context distributions with equal frequency (UA-1, UA-2). The EMA rule is a
contraction mapping with rate $\lrContent$: at each step, $z_i$ moves a fraction
$\lrContent$ toward $\bar{x}_i$. By UA-3, the time-averaged input to each unit
converges to the mixture centroid $\bar{c}$. Since all units receive the
same mixture, inter-unit variance decays exponentially as
$(1-\lrContent)^t \to 0$, yielding convergence to $\bar{c}$.
\end{proof}

Proposition~\ref{thm:t3} explains why na\"{\i}ve dissipative architectures
cannot maintain context-specific representations: under uniform activation,
the content update rule is a homogenizing force that erases inter-context
distinctions.

\subsection{Discrete Routing Breaks Convergence}
\label{sec:moe-theory}

\begin{proposition}
\label{prop:moe-breaks-t3}
Under discrete expert routing with $K$ groups, each unit receives updates
from at most one context. The contraction in Proposition~\ref{thm:t3} is
restricted to each expert subpopulation: content vectors within group $k$
converge to the context-specific centroid $c_k$, not the grand centroid
$\bar{c}$. The inter-expert variance $\operatorname{Var}(\{c_k\})$ is
preserved.
\end{proposition}

\begin{proof}[Proof sketch]
Under discrete routing a unit $i \in G_k$ is updated only on steps where
expert $k$ is selected, so the sequence of inputs it averages is drawn from
context $k$ alone rather than from the mixture. Applying the argument of
Proposition~\ref{thm:t3} within $G_k$ gives convergence to $c_k$. The groups
are disjoint and no update couples them, so the between-group variance
$\operatorname{Var}(\{c_k\})$ is not contracted.
\end{proof}

This proposition identifies what a correct assignment buys: without discrete
routing, the homogenizing force of the content update rule renders all memory
entries identical regardless of the memory mechanism employed. Two
qualifications bound what it establishes.

First, the proposition \emph{assumes} routing is correct --- that inputs from
context $k$ reach group $k$ --- and says nothing about how such an assignment
is obtained. Expert assignment is itself an inverse problem, and an ill-posed
one: many partitions of $N$ units into $K$ groups are consistent with a given
input stream, and the mapping from partitions to observable behaviour is not
injective. Proposition~\ref{prop:moe-breaks-t3} is therefore a statement about
what a correct assignment buys, not a construction of one. The experiments of
Section~\ref{sec:res-causality} are shaped accordingly: they test the
consequence of \emph{destroying} an assignment, which is well posed, rather
than the uniqueness of recovering one, which is not.

Second, the two propositions describe idealised extremes --- perfectly uniform
activation on one side, perfectly stable routing on the other. Neither holds
exactly in the running system, in which stochastic replacement continuously
perturbs the content held by each group's units. The regime the architecture actually occupies lies
between them, and is the subject of the next section.

\subsection{From Routing to Persistence: The DM Bridge}
\label{sec:bridge}

Propositions~\ref{thm:t3} and~\ref{prop:moe-breaks-t3} characterize two
extremes: uniform activation leads to collapse, and perfect discrete
routing preserves separation. The full DM system operates in an
intermediate regime in which stochastic turnover continuously perturbs
the content held by each expert's units, while leaving the partition itself
intact. The following proposition connects routing structure
to the observed persistent memory:

\begin{proposition}[DM Seeding Maintains Effective Routing Under Turnover]
\label{prop:bridge}
Under stochastic unit replacement at rate $\turnover$, the injection of
stored expert centroids $m_k$ into replacement units (Loop~2) biases the
content of new units toward their expert's attractor. Combined with
input-driven expert selection, this produces an effective update
distribution where each unit $i$ in group $k$ receives updates
predominantly from context $k$, violating Assumption~\textup{UA-1}. The
resulting mutual information $\operatorname{MI}(I; K) > 0$ mediates the
representation quality $R$.
\end{proposition}

This proposition yields two observable signatures that constitute its
falsification criteria:
\begin{enumerate}
  \item \textbf{Routing specialization (S1)}: With DM active,
    $\operatorname{MI}(I; K)$ should increase significantly compared to
    random-seeded controls. If $\operatorname{MI}$ does not increase, the
    bridge from discrete routing to persistent memory is broken.
  \item \textbf{Mediation (S2)}: Representation quality $R$ should
    correlate with binding quality across conditions (Spearman
    $\rho(R, \text{MI}) > 0$). If $R$ improves without corresponding
    $\operatorname{MI}$ increase, then Proposition~\ref{prop:moe-breaks-t3}
    is not the operative mechanism.
\end{enumerate}

\subsection{Triple-Loop as Fixed-Point Maintenance}
\label{sec:dm-theory}

Under stochastic turnover with rate $\turnover$, the expected lifetime of any
individual unit's content is $\sim 1/\turnover$ steps. Without DM, a unit's
content is drawn from the correct expert centroid only if it has survived
long enough to converge (requiring $\gg 1/\lrContent$ steps). With DM seeding at
injection rate $\pInject$, replaced units start at $m_k \approx c_k$ rather than
at random, effectively resetting the convergence clock.

The triple-loop creates a population-level fixed point: even as individual
units are destroyed and replaced, the distribution of content vectors within
each expert group remains centered on $m_k$. This fixed point is maintained
when the seeding rate exceeds the dissipative drift:
$\pInject \cdot \turnover > \alpha_{\text{drift}}$,
where $\alpha_{\text{drift}}$ is the effective content drift rate.

The condition above is heuristic: $\alpha_{\text{drift}}$ is not defined
independently of the quantity it is meant to bound, and the symbol collides
with the content rate $\lrContent$. We therefore state the maintenance
property directly, as a drift bound on an explicit Lyapunov candidate.

Let
\begin{equation}
  V_t \;=\; \sum_{k=1}^{K}\;\sum_{i \in G_k}
            \bigl\| z_i^{(t)} - m_k^{(t)} \bigr\|^{2}
  \label{eq:lyap}
\end{equation}
denote the total within-expert dispersion of content vectors about their
recorded centroids. $V_t$ is a function of the internal state alone; no
ground-truth centroid enters it.

\begin{proposition}[Triple-loop drift bound]
\label{prop:drift}
Assume \textup{(i)} the expert partition $\{G_k\}$ is fixed over the step
(discrete routing, Section~\textup{\ref{sec:moe}}); \textup{(ii)} each unit is
replaced independently with probability $\turnover$; \textup{(iii)} the
re-initialisation $\xi$ has finite second moment; and \textup{(iv)}
$\lrContent \in [0,1]$ and $\lrRecord \in (0,1]$. Write $A_t$ for the units
active at step $t$, $k^{*}$ for the selected expert, $n = \lvert G_{k^*}\rvert$,
and
\[
  V_t^{A} = \!\!\sum_{i \in A_t}\!\lVert z_i - m_{k(i)}\rVert^{2},\quad
  S_t = \!\!\sum_{i \in A_t}\!\lVert \bar{x}_i - m_{k(i)}\rVert^{2},\quad
  b_t = \bar{z}_{F_{k^*}} - \bar{z}_{G_{k^*}},\quad
  D_t^{+} = \sum_{i}\mathbb{E}\lVert \xi - m^{+}_{k(i)}\rVert^{2},
\]
where $b_t$ is the \emph{firing-set bias} --- the offset between the mean
content of the units that fired and the mean over the whole group --- and
$m^{+}$ is the centroid \emph{after} the recording step. Then
\begin{equation}
  \mathbb{E}\bigl[V_{t+1} \mid \mathcal{F}_t\bigr]
  \;\le\;
  (1-\turnover)\Bigl(V_t - \lrContent V_t^{A} + \lrContent S_t
    + \lrRecord\, n\, \lVert b_t\rVert^{2}\Bigr)
  \;+\; \turnover\,(1-\pInject)\,D_t^{+}.
  \label{eq:drift}
\end{equation}
\end{proposition}

\begin{proof}[Proof sketch]
The loops act in the order given by Algorithm~\ref{alg:cycle}: content
update, then recording, then replacement.

\emph{Content update} is a convex combination for $\lrContent \in [0,1]$, so
$\lVert(1-\lrContent)z_i + \lrContent\bar{x}_i - m_k\rVert^{2}
 \le (1-\lrContent)\lVert z_i - m_k\rVert^{2}
   + \lrContent\lVert \bar{x}_i - m_k\rVert^{2}$,
contributing $-\lrContent V_t^{A} + \lrContent S_t$ on the active set.

\emph{Recording} moves $m_{k^*}$ toward the mean of the \emph{firing} units,
not of the whole group. Writing $\delta = \bar{z}_{G_{k^*}} - m_{k^*}$ and
$b = b_t$, so that $\bar{z}_{F_{k^*}} - m_{k^*} = \delta + b$, substitution
into Eq.~\eqref{eq:lyap} gives exactly
\[
  \Delta V
  = -\lrRecord(2-\lrRecord)\,n\lVert\delta\rVert^{2}
    - 2\lrRecord(1-\lrRecord)\,n\langle \delta, b\rangle
    + \lrRecord^{2}\, n \lVert b\rVert^{2}.
\]
The cross term has no fixed sign, so recording is \emph{not} in general a
descent step on $V$: when the firing set is unrepresentative of its group,
recording can increase the dispersion. Applying Young's inequality to the
cross term with unit weight and using $\lrRecord \le 1$ gives
\[
  \Delta V \;\le\; -\lrRecord\, n \lVert\delta\rVert^{2}
                   + \lrRecord\, n \lVert b\rVert^{2},
\]
so recording has the same contraction-plus-injection form as the other two
loops, with the firing-set bias supplying its injection term. Discarding the
non-positive contraction leaves the $\lrRecord\, n \lVert b_t\rVert^{2}$ term
in Eq.~\eqref{eq:drift}.

\emph{Replacement} deletes the term of each replaced unit and redraws it: with
probability $\pInject$ the unit is set to $m^{+}_{k}$ and contributes zero,
otherwise it contributes $\mathbb{E}\lVert \xi - m^{+}_{k}\rVert^{2}$. Taking
expectations over the replacement indicators gives the factor $(1-\turnover)$
on surviving terms and $\turnover(1-\pInject)D_t^{+}$ on redrawn ones. Note
that $D_t^{+}$ is measured against the post-recording centroid, which is the
one seeded units actually receive.
\end{proof}

\begin{corollary}[Seeding rescales the turnover noise floor]
\label{cor:floor}
Suppose in addition that the active set carries at least a fraction $\beta$ of
the dispersion, $V_t^{A} \ge \beta V_t$; that $S_t \le S$, $D_t^{+} \le D$ and
$n\lVert b_t\rVert^{2} \le B$ uniformly in $t$; and that
$q = (1-\turnover)(1-\lrContent\beta) < 1$. Then
\begin{equation}
  \limsup_{t \to \infty} \mathbb{E}[V_t]
  \;\le\;
  \frac{(1-\turnover)\bigl(\lrContent S + \lrRecord B\bigr)
        \;+\; \turnover\,(1-\pInject)\,D}{1 - q}.
  \label{eq:floor}
\end{equation}
\end{corollary}

Three features of Eq.~\eqref{eq:floor} bear on the experiments that follow.
First, the turnover term enters multiplied by $(1-\pInject)$: seeding does not
remove the perturbation caused by replacement, it rescales it, and as
$\pInject \to 1$ that term vanishes. This is the sense in which the loop
maintains a fixed point --- not by preventing turnover, but by setting what
turnover costs. Second, recording carries a floor of its own, $\lrRecord B$,
governed by how unrepresentative the firing set is of its group: recording
stabilises only to the extent that the units which fire speak for the units
that do not. Third, the two terms are not obtainable from one another, which
is the analytical counterpart of the two loops' non-interchangeability in the
ablation of Section~\ref{sec:res-ablation}.

We verified Eq.~\eqref{eq:drift} numerically over a grid of
$(\lrContent, \lrRecord, \turnover, \pInject)$ values, drawing the firing set
as a strict random subset of the active group at each step and evaluating the
conditional expectation by resampling the replacement indicators; the
inequality held at every step in every cell tested. As a positive control we
checked that the same procedure rejects the bound one obtains by treating
recording as if it used the whole-group mean: a two-unit group with a single
firing unit and $\lrRecord = 1$ raises $V$ from $50$ to $100$, while that
bound predicts a decrease. The $(1-\pInject)$ scaling predicted by
Eq.~\eqref{eq:floor} is present in the measured data: across the $(K,p)$ sweep
of Section~\ref{sec:res-envelope}, dispersion decreases monotonically as
$\pInject$ increases (Spearman $\rho = -0.86$, $p < 10^{-15}$, $n = 50$).

\paragraph{What the bound does not establish.} Eq.~\eqref{eq:drift} is a
Foster--Lyapunov drift condition, not a complete stability theory. It bounds
the expected dispersion; it does not construct an invariant measure, prove
ergodicity, or characterise the transient. Assumption \textup{(i)} in
particular excludes the regime in which routing itself is unstable, which is
where we would expect the analysis to be hardest.

This framework yields three falsifiable predictions:
\begin{enumerate}
  \item If per-cycle random permutation of expert assignments (destroying
    binding while preserving partition structure) produces $\text{MI}$
    comparable to stable binding, then binding is \emph{not} the causal
    mechanism.
  \item If one-shot seeding (a single injection at the start) achieves
    reconstruction quality comparable to continuous seeding, then the
    closed-loop property is \emph{not} necessary.
  \item If representation quality $R$ does not vary systematically with
    $(K, \pInject)$, then the operating envelope claim is falsified.
\end{enumerate}

\section{Results}
\label{sec:results}

\noindent\textit{Statistical scope.} Sections~\ref{sec:res-dm1}--\ref{sec:res-ablation}
report pilot-scale blocks with 4--5 seeds per condition. As set out in
Section~\ref{sec:protocol}, their intervals describe the spread of the observed
seeds and their $p$-values are descriptive: at $n = 5$ the smallest attainable
two-sided Wilcoxon $p$ is $0.0625$ (at $n = 4$, $0.125$), so none of these
comparisons can reach $p < 0.05$ regardless of effect size. In the confirmatory
block (Section~\ref{sec:res-p1}) the single-factor ablation uses 14 seeds per
condition; seed counts for the other confirmatory experiments are given per
experiment in the Supplementary Material.

\subsection{Removing Discrete Routing Removes Specialization}
\label{sec:res-causality}

To test whether stable context--expert binding is necessary for
structural specialization, we compare
three conditions: (1)~\textbf{Full MoE}: stable deterministic expert
assignment; (2)~\textbf{Baseline}: no expert routing (all units active in
all contexts); and (3)~\textbf{Binding disruption} (per-cycle random
permutation): expert groups are randomly reassigned at each computation
cycle. This multi-factor stress test ($n = 91$ runs across two experimental
blocks---a routing probe, 56 runs, and a binding-disruption sham, 35 runs---with
7 seeds per condition) simultaneously disrupts temporal binding, evidence
accumulation, and centroid tracking. It serves as a \emph{necessary
condition test}: if specialization survives binding disruption, stable
binding would be unnecessary. It is not sufficient for causal
identification, as the intervention perturbs multiple coupled mechanisms.

Full MoE produces strong structural specialization: firing selectivity
$f_{\text{sel}} = 0.959$ (vs.\ $0.051$ at baseline; $\Delta = +0.908$),
silhouette score $0.166$ (vs.\ $0.121$; $\Delta = +0.045$), and
representation divergence $\Delta R = 0.467$ (vs.\ $-0.092$).
Under binding disruption, mutual information
drops from $\text{MI} = 1.099$ (Full; equal to the theoretical maximum
$\ln K$) to $\text{MI} = 0.001$, confirming that the disrupted system
cannot maintain context--expert information
(Figure~\ref{fig:causality}).

The disrupted condition's silhouette ($0.118$) is statistically equivalent
to baseline ($0.121$; equivalence margin $\epsilon = 0.009$), confirming
that partition structure without stable binding provides no specialization
benefit. We conclude that \emph{stable context--expert binding is required}
for structural specialization in this architecture. Full causal
identification---isolating binding from co-occurring mechanisms such as
evidence accumulation and load balancing---requires single-factor ablations,
which we identify as a priority for future confirmatory work.

\begin{figure}[t]
  \centering
  \includegraphics[width=\linewidth]{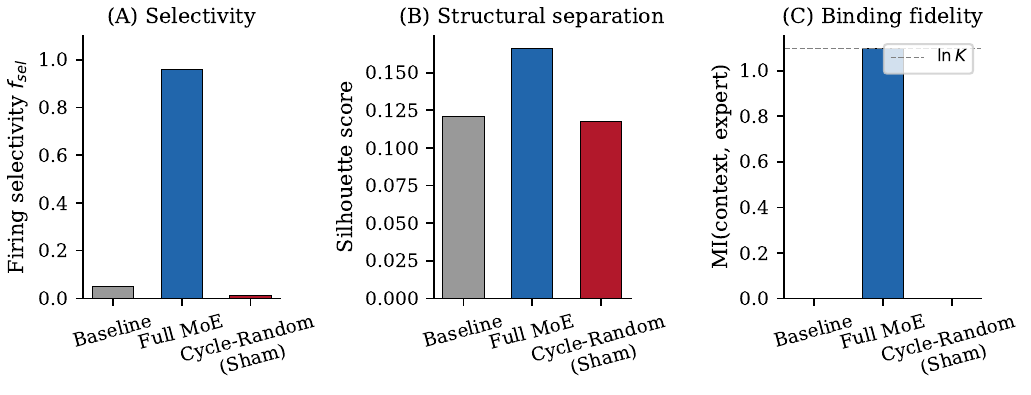}
  \caption{Removing discrete routing removes structural
    specialization. (A)~Firing selectivity: Full MoE achieves near-perfect
    context-exclusive firing ($f_{\text{sel}} = 0.959$); binding disruption
    collapses to baseline. (B)~Structural separation via silhouette score.
    (C)~Mutual information between context and expert assignment: Full MoE
    achieves the theoretical maximum ($\ln K$); per-cycle random permutation
    eliminates binding ($\text{MI} = 0.001$). $n = 91$ runs, 7 seeds per condition.}
  \label{fig:causality}
\end{figure}

\subsection{Deep Memory Creates Persistent Representations}
\label{sec:res-dm1}

Having established what a correct expert assignment buys, we test
whether DM produces persistent, context-specific representations that
survive stochastic turnover. Four conditions are compared ($n = 16$ runs,
4 seeds): (1)~\textbf{Full DM} with correct expert--context mapping;
(2)~\textbf{Global control} without expert-specific routing;
(3)~\textbf{Mismatched write} where recording uses the wrong expert
mapping; and (4)~\textbf{Noise write} where centroids are corrupted before
storage.

Full DM achieves $R = 0.984$ with remarkable consistency across all
contexts ($R_{c_0} = 0.984$, $R_{c_1} = 0.984$, $R_{c_2} = 0.985$).
The global control, which records a single centroid across all contexts,
produces $R = 0.385$---marginally above chance and far below the DM
condition ($\Delta = +0.599$). These results are displayed in
Figure~\ref{fig:dm1}.

The failure modes provide mechanistic insight. Mismatched write produces
$R = -0.003$, indicating that incorrect expert--centroid mapping causes
\emph{catastrophic} failure rather than graceful degradation: reborn units
are seeded with the wrong context's content, actively interfering with
representation maintenance. Noise write produces $R = 0.935$, showing
partial degradation proportional to corruption magnitude. This
asymmetry---catastrophic failure under mapping error versus graceful
degradation under noise---is a \emph{mechanism fingerprint} consistent with
expert-specific recording being the functional locus.

\begin{figure}[t]
  \centering
  \includegraphics[width=0.65\linewidth]{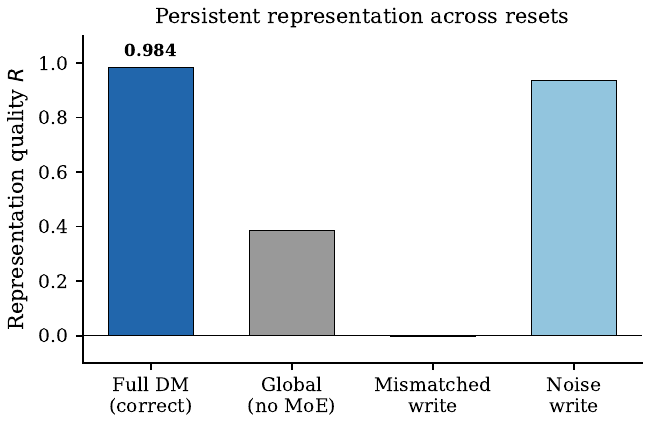}
  \caption{Deep Memory creates persistent representations. Full DM with
    correct expert mapping achieves $R = 0.984$, consistent across all
    contexts. Global control ($R = 0.385$) lacks context specificity.
    Mismatched write ($R \approx 0$) produces catastrophic failure;
    noise write ($R = 0.935$) degrades gracefully. $n = 16$ runs,
    4 seeds.}
  \label{fig:dm1}
\end{figure}

\subsection{Continuous Seeding Enables Functional Reconstruction}
\label{sec:res-dm2}

DM's utility depends not only on maintaining representations during
stable operation, but on \emph{reconstructing} them after disruption.
We test this by introducing an interference phase that degrades
representations, then measuring recovery under six conditions
($n = 30$ runs, 5 seeds): (1)~\textbf{Continuous seeding};
(2)~\textbf{Noise seeding} (random vectors instead of stored centroids);
(3)~\textbf{One-shot seeding} (single injection at reconstruction start);
(4)~\textbf{Baseline} (no DM); (5)~\textbf{Wrong-expert seeding}
(centroids from mismatched experts); and (6)~\textbf{Recording off} (seeding
continues, but the stored centroids are never updated from their random
initial values).

Continuous seeding achieves near-complete reconstruction:
$R_{\text{recon}} = 0.978$ (95\% CI $[0.977, 0.979]$), recovering from a post-interference baseline
of $R = 0.290$ ($\Delta = +0.688$). This confirms that DM functions as a
restorative mechanism, not merely a maintenance one
(Figure~\ref{fig:dm2}).

The critical negative control is one-shot seeding: $R_{\text{recon}} =
0.305$ (95\% CI $[0.280, 0.314]$), statistically indistinguishable from the no-DM baseline. This
confirms that a single injection is insufficient---the content averaging
rule washes out any one-time seed within approximately $1/\turnover$ turnover
cycles. \emph{One-shot seeding does not suffice.}

Noise seeding achieves intermediate performance ($R_{\text{recon}} = 0.712$, 95\% CI $[0.705, 0.722]$),
demonstrating a content-specificity gap of $\Delta = 0.266$ (95\% CI $[0.258, 0.272]$, paired by seed) between
DM-seeded and noise-seeded conditions. This gap confirms that stored
centroids carry discriminative information beyond generic initialization.
Wrong-expert seeding produces $R_{\text{recon}} = 0.125$, \emph{worse}
than baseline ($\Delta = -0.165$), confirming that incorrect expert--content
alignment is actively harmful. Recording off gives a similar value
($R_{\text{recon}} = 0.126$): seeding without recording delivers no stored
content.

\begin{figure}[t]
  \centering
  \includegraphics[width=0.75\linewidth]{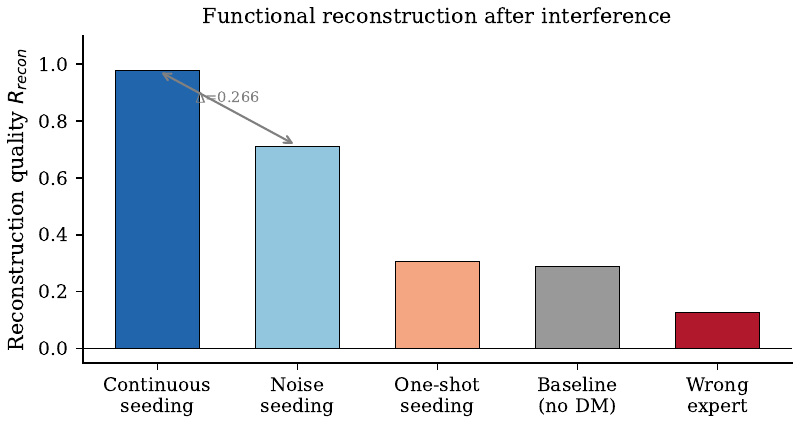}
  \caption{Functional reconstruction requires continuous seeding and
    expert--content alignment. Continuous seeding achieves near-complete
    recovery ($R_{\text{recon}} = 0.978$); one-shot seeding fails
    ($R_{\text{recon}} = 0.305 \approx$ baseline). Noise seeding recovers
    partially ($0.712$; content-specificity gap $= 0.266$). Wrong-expert
    seeding is harmful ($0.125 <$ baseline). $n = 30$ runs, 5 seeds.}
  \label{fig:dm2}
\end{figure}

\subsection{Operating Envelope and Degradation Boundaries}
\label{sec:res-envelope}

We use \emph{regime} and \emph{boundary} here in an operational sense. The
regions reported below are separated by thresholds on measured
representation quality, not by an order parameter or by a demonstrated
change in the qualitative character of the dynamics; we make no claim of a
phase transition in the statistical-mechanical sense.

DM's practical utility requires understanding \emph{where} it works and
where it fails. We sweep the expert count $K$ and injection rate $\pInject$
across 10 operating points ($n = 115$ runs, 5 seeds), classifying each
as \textbf{Pass} ($R > 0.70$ and all four quality criteria met),
\textbf{Degraded} (positive but reduced $R$), or \textbf{Fail}
(insufficient lift over baseline).

Figure~\ref{fig:envelope} shows the results. At full injection ($\pInject = 1.0$),
DM maintains high quality across expert counts: $R = 0.980$ at $K = 3$,
$R = 0.964$ at $K = 5$, and $R = 0.882$ at $K = 8$. Reducing injection
rate reveals a degradation boundary: $\pInject = 0.25$ is sufficient at $K = 3$
($R = 0.837$, Pass), but $\pInject = 0.10$ produces degraded output
($R = 0.568$) and $\pInject = 0.05$ degrades further ($R = 0.459$). The single
Fail condition ($K = 3$, block size $= 10$, $\pInject = 0.1$) has delta too
small to distinguish from baseline despite acceptable absolute $R$.

Overall: 7/10 conditions Pass, 2/10 are Degraded but functional, 1/10
Fails. Higher expert counts require correspondingly higher injection rates
to maintain quality, with the degradation boundary depending on the
interaction between $K$ and $\pInject$.

\begin{figure}[t]
  \centering
  \includegraphics[width=0.8\linewidth]{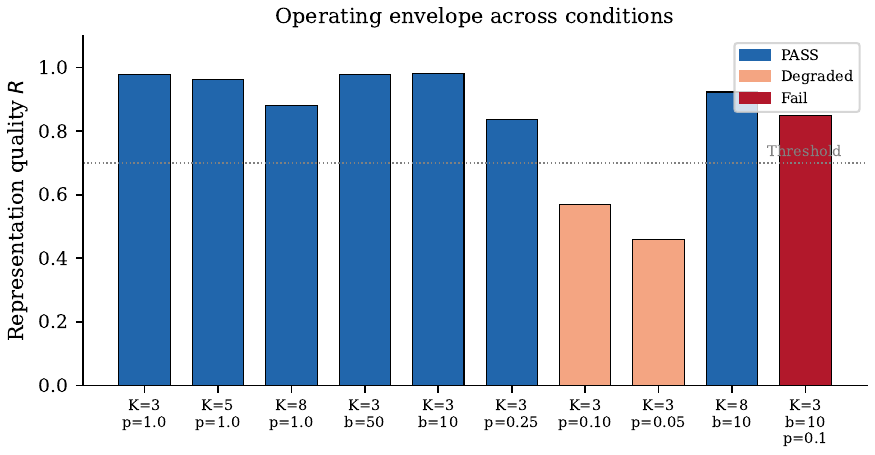}
  \caption{Operating envelope across expert count $K$, injection rate $\pInject$,
    and block size. Blue: Pass (7/10); orange: Degraded (2/10); red: Fail
    (1/10). Low $p$ and high $K$ define the degradation boundary. The
    threshold line at $R = 0.70$ separates Pass from sub-threshold
    regimes. $n = 115$ runs, 5 seeds.}
  \label{fig:envelope}
\end{figure}

\subsection{Scheduling Invariance and the $K \times p$ Regime Map}
\label{sec:res-scheduling}

A memory mechanism useful in practice must be robust to the temporal
structure of context presentation. We test DM under five qualitatively
different scheduling patterns ($n = 75$ runs, 5 seeds): uniform block,
Markov sticky (high self-transition probability), random i.i.d., bursty
Zipf (power-law context frequency), and session restart (periodic grid
reinitialization).

All five schedulers produce Pass-quality representations, with $R_{\text{recon}}$
ranging from $0.83$ to $0.89$ (Figure~\ref{fig:robustness}A). The
no-DM baseline shows much greater scheduler sensitivity ($R$ range:
$0.28$--$0.56$), indicating that DM not only improves absolute quality but
also \emph{stabilizes} it across scheduling regimes.

A separate $K \times p$ regime map ($n = 235$ runs; $K \in \{3, 5, 8\}$,
$p \in \{0.1, 0.25, 0.5\}$, two schedulers) confirms the envelope
boundaries at larger scale (Figure~\ref{fig:robustness}B). At $p \geq 0.5$,
all $(K, \pInject)$ cells are Pass regardless of scheduler or expert count. At
$\pInject = 0.25$, cells with $K \leq 5$ remain Pass while $K = 8$ degrades.
At $\pInject = 0.10$ with $K \geq 8$, the system Fails. This structure is
stable across both schedulers tested, confirming that the operating
boundaries are intrinsic to the $(K, \pInject)$ interaction rather than artifacts
of a specific scheduling regime.

\begin{figure}[t]
  \centering
  \includegraphics[width=\linewidth]{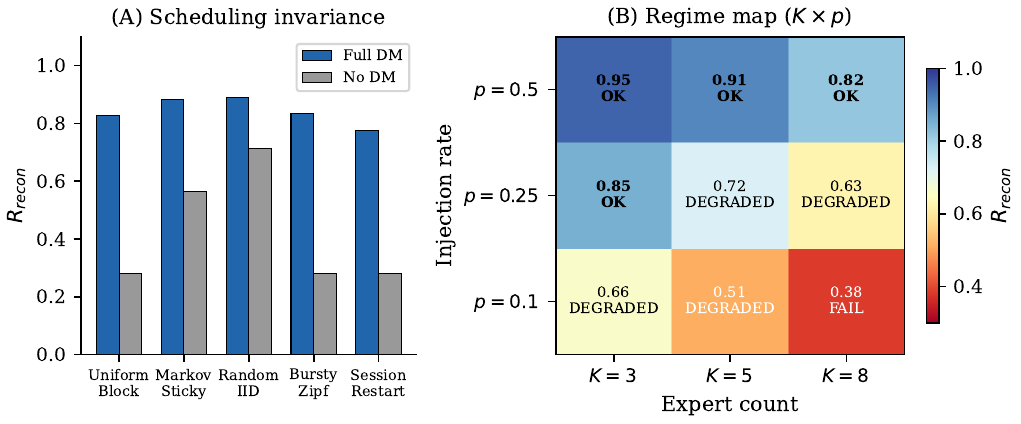}
  \caption{DM is robust to scheduling and exhibits a stable regime map.
    (A)~Scheduling invariance: Full DM (blue) achieves Pass quality under
    all five scheduling patterns; no-DM baseline (gray) shows high
    variability. $n = 75$ runs, 5 seeds.
    (B)~Regime map across $K$ and $p$: dark cells indicate high $R$
    (Pass); light cells indicate degradation or failure. The boundary
    the interaction between $K$ and $\pInject$ is consistent across schedulers.
    $n = 235$ runs.}
  \label{fig:robustness}
\end{figure}

\subsection{Component Ablation}
\label{sec:res-ablation}

To identify the minimal mechanism, we conduct a $2^3$ factorial ablation
over three binary factors ($n = 40$ runs, 5 seeds): recording (DM centroid
update), seeding (centroid injection at replacement), and anchoring
(continuous content pull toward the stored centroid). This yields eight
conditions from no-DM to full DM.

Figure~\ref{fig:ablation} reveals a sharp pattern: the recording $\times$
seeding dyad is the critical mechanism. When both are active, $R = 0.854$
($\Delta = +0.533$ over baseline $R = 0.321$). When either is removed,
performance collapses to baseline: recording-only, seeding-only, and no-DM
all produce $R = 0.321$.

The recording--seeding dyad is
not decomposable in this design: neither component alone is sufficient, but together
they produce the full DM effect. Recording without seeding accumulates
centroids that are never deployed; seeding without recording injects
uninitialized (random) centroids. The anchor component did not reach
significance as an independent factor in the current experimental
regime, though its contribution under alternative operating conditions
remains an open question.

\begin{figure}[t]
  \centering
  \includegraphics[width=0.75\linewidth]{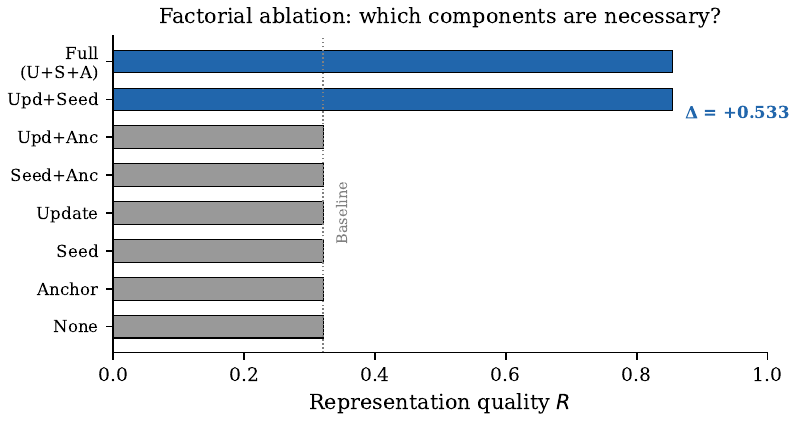}
  \caption{Factorial ablation isolates recording $\times$ seeding as the
    minimal critical dyad. Only conditions with both recording and seeding
    active (bottom two bars) exceed baseline. Single factors and the
    anchor component provide no lift. $\Delta = +0.533$ for the critical
    dyad over baseline. $n = 40$ runs, 5 seeds.}
  \label{fig:ablation}
\end{figure}

\subsection{Confirmatory Ablation and External Baselines}
\label{sec:res-p1}

Having established the core mechanism (Sections~\ref{sec:res-causality}--\ref{sec:res-ablation}),
we conduct six confirmatory experiments addressing three open questions:
single-factor causal isolation, mechanism bridging across scales,
and comparison with associative and reservoir baselines under matched turnover.
Table~\ref{tab:p1-summary} summarizes the design and primary outcomes.

\begin{table}[t]
\centering
\caption{Summary of confirmatory experiments ($n = 370$ runs total).
  $R$: representation fidelity (cosine alignment with ground-truth prototypes).
  $S_{\text{ctx}}$: context silhouette (cosine distance), measuring functional
  separability of representations grouped by context identity.
  Full statistics and per-seed results are provided in the Supplementary Material.}
\label{tab:p1-summary}
\small
\begin{tabular}{@{}llccc@{}}
\toprule
Expt. & Purpose & Runs & Primary metric & Verdict \\
\midrule
E1 & Single-factor causal ablation & 112 & $R$ & Confirmed \\
E2 & Mediation: $S_{\text{ctx}} \to R$ & 60 & $\beta$, $\rho$ & Confirmed \\
E3 & $K$-scaling ($K{=}5, 8$) & 54 & $R$, $S_{\text{ctx}}$ & Confirmed \\
E4 & Turnover-dose $\times$ DM & 36 & $\Delta R$, $\Delta S_{\text{ctx}}$ & Partial \\
E5 & Soft vs.\ hard routing & 48 & $S_{\text{ctx}}$ & Partial \\
E6 & Associative/reservoir baselines & 60 & $R$, $S_{\text{ctx}}$ & Confirmed \\
\bottomrule
\end{tabular}
\end{table}

\paragraph{Single-factor causal ablation.}
Section~\ref{sec:res-ablation} identified the recording--seeding dyad through
factorial ablation. To isolate each component's individual contribution, we
conduct single-factor ablation across eight conditions ($K = 3$, 14 seeds
each, $n = 112$). Five key results emerge: (1)~Full DM achieves
$R = 0.983$; (2)~removing DM entirely yields $R = 0.296$; (3)~disabling
seeding alone (\textsc{Seed\_off}) also yields $R = 0.296$, identical to
no-DM in every seed (all per-seed differences are zero), so that
\emph{seeding is the sole delivery channel}---recording and anchoring alone
are inert under turnover;
(4)~seeding with mismatched expert centroids (\textsc{DM\_wrong}) yields
intermediate performance ($R = 0.557$), demonstrating content-specificity;
(5)~replacing context-bound expert assignment by context-independent expert
cycling reduces $R$ to $0.804$, so stable routing is necessary but not
sufficient. Two further conditions locate the source of the damage: without
turnover DM reaches $R = 0.997$, and freezing threshold adaptation leaves
$R$ unchanged ($0.983$). Per-condition values are given in Supplementary
Table~\ref{tab:s1-e1}.

\paragraph{Functional separability as mediator.}
A natural question is whether DM improves $R$ directly or by maintaining
an intermediate structural property. We define $S_{\text{ctx}}$ (context
silhouette; cosine distance) as a prototype-free measure of functional
separability: expert-level content representations are scored via
silhouette analysis, labelled by expert; because the context--expert map is
fixed in these experiments, this is equivalent to grouping by ground-truth
context identity. Across five injection doses plus a noise-injection control
($K = 3$, $n = 60$), $S_{\text{ctx}}$ and $R$ are perfectly rank-ordered
($\rho = 1.0$ in every seed; rank-regression coefficient of $R$ on
$S_{\text{ctx}}$ controlling for dose and seed $\beta = 0.909$, $t = 14.9$).
Critically, noise-seeding (injecting random centroids at the same rate)
produces $S_{\text{ctx}}$ and $R$ indistinguishable from the no-injection
baseline, confirming that content-specific centroids---not injection
mechanics---drive the effect. This relationship generalizes across scales:
at $K = 5$ and $K = 8$ ($n = 54$), the DM advantage is reproduced
(one-sided sign test $p = 0.016$), and $S_{\text{ctx}}$ remains the primary mediator.
Because routing is hard and fixed in this regime, the mediation path
operates through content-state separability rather than routing mutual
information.
Two additional experiments (E4: turnover-dose interaction, E5: soft
vs.\ hard routing) further characterize the mechanism: the DM effect on
$R$ is the same at the two turnover levels tested ($\Delta R \approx 0.75$),
while the separability effect ($\Delta S_{\text{ctx}}$) differs between them,
indicating dual channels of action. Removing the routing mask entirely
degrades $S_{\text{ctx}}$ through cross-expert interference, whereas a small
mask leak leaves it essentially unchanged; details are provided in
Supplementary~S3--S4.

\paragraph{Associative and reservoir baselines under matched turnover.}
To test whether DM's persistence advantage is specific to the dissipative
architecture or achievable by simpler memory systems, we compare against
two alternative memory baselines with matched stochastic turnover at rate
$\turnover$: (1)~a modern Hopfield network \citep{ramsauer2021hopfield}
with softmax-attention retrieval and Hebbian storage, subject to
per-slot probabilistic reset; and (2)~an echo state network (ESN)
\citep{jaeger2001esn} with fixed random reservoir and adaptive linear
readout, subject to per-neuron state and readout weight reset.
Turnover rates are calibrated to match the empirical death rate of the
dissipative grid under standard operating conditions.

Turnover destroys Hopfield memory: $R$ drops from $0.999$ (no turnover)
to $0.474$ (matched turnover), confirming that classical associative
memory is fragile under stochastic replacement.
The full DM system significantly outperforms the Hopfield baseline under
identical turnover ($R = 0.854$ vs.\ $0.474$, $p \approx 0.001$, one-sided
Wilcoxon, 10 seeds), and the
DM advantage is confirmed by the within-architecture contrast
($\text{Full DM} - \text{No DM} = +0.749$).
The ESN achieves higher $R$ ($0.943$) due to its supervised readout
objective; however, it uses a fundamentally different learning paradigm
(LMS weight update with direct error signal) and thus serves as a
diagnostic upper bound rather than a fair comparison. Notably, on
functional separability ($S_{\text{ctx}}$), the unsupervised DM system
exceeds the supervised ESN ($0.900$ vs.\ $0.857$), suggesting that DM's
expert-based organization provides structural advantages beyond raw
reconstruction accuracy.
Full implementation details and the ESN regularization analysis are
provided in Supplementary~S5.

\paragraph{What is and is not matched in these comparisons.} Turnover rate
is matched by construction. Three axes are not, and we state the direction
of each. \emph{Learning paradigm}: the ESN readout is fitted by supervised
least-mean-squares against a target, whereas DM is given no target; this
favours the ESN, which is why we read it as an upper bound rather than a
competitor. \emph{Parameter budget}: the three systems allocate storage
differently---DM holds $K$ centroids of dimension $D$, the Hopfield network
holds a slot per stored pattern, and the ESN holds a reservoir together
with a readout matrix---and we did not equalise total stored parameters, so
cross-architecture magnitudes should not be read as a like-for-like
ranking. \emph{Retrieval objective}: Hopfield retrieval minimises an energy
over stored patterns, while DM re-injects a running summary; the two are
not solving an identical task.

For this reason the load-bearing comparison here is the
\emph{within-architecture} contrast (Full DM $-$ No DM $= +0.749$), in
which capacity, paradigm and objective are identical by construction and
only the DM loops differ. The cross-architecture numbers locate the result
relative to familiar memory models; they are not offered as a ranking.

\section{Discussion}
\label{sec:discussion}

\subsection{A Functional Analogy to Hippocampal Consolidation}
\label{sec:hippocampal}

The comparison below is functional only. DM and hippocampal
consolidation as described in complementary learning systems theory
\citep{mcclelland1995cls, kumaran2016hippocampus} can be compared at the
level of the problem each solves, not at the level of mechanism, and the
comparison holds only with the fast--slow division reversed
(Section~\ref{sec:dm}). With that qualification, three DM components can be
set beside three hippocampal functions:

\begin{enumerate}
  \item \textbf{Slow integration $\leftrightarrow$ Recording.} In CLS the
    hippocampus encodes experiences rapidly in sparse, pattern-separated
    representations, while the neocortex integrates them gradually into
    structure that outlasts any single episode \citep{squire1992memory}.
    DM recording occupies the latter role, not the former: the centroid
    $m_k$ accumulates a population-level summary through an exponential
    moving average (Eq.~\ref{eq:recording}), so it is the slow term, while
    the continually replaced active units are the fast one.
    Section~\ref{sec:dm} states this inversion, and its cause---the absence
    of a gradient-based cortical learner---explicitly.
  \item \textbf{Replay and reactivation $\leftrightarrow$ Seeding.}
    Hippocampal sharp-wave ripples reactivate stored patterns during
    quiescence \citep{buzsaki2015hippocampal, carr2011hippocampal}. DM
    seeding re-injects stored centroids into replacement units at each
    turnover event, functionally analogous to reactivation.
  \item \textbf{Pattern separation $\leftrightarrow$ MoE partitioning.}
    The dentate gyrus maintains distinct representations for similar inputs
    through sparse, non-overlapping coding \citep{yassa2011pattern}. MoE
    discrete routing enforces non-overlapping expert groups, preventing
    the interference that would collapse representations to a single
    centroid (Proposition~\ref{thm:t3}).
\end{enumerate}

This is a \emph{functional computational analogy}, not a claim of
biological mechanism: DM and hippocampal consolidation address a similar
functional problem---persistent memory under interference and
turnover---and we make no claim that they do so in the same way. The
stochastic unit replacement in our architecture resembles adult
neurogenesis in the dentate gyrus \citep{aimone2014regulation} in the same
limited, functional sense only.

Read at this level, three of our results have qualitative counterparts in
the hippocampal literature. We list them as correspondences, not as
confirmed predictions of the analogy:
\begin{itemize}
  \item \textbf{Mapping-error catastrophe}: Hippocampal lesion
    studies report that disrupting CA3--CA1 projections produces severe
    rather than graded memory failure \citep{squire1992memory}; similarly, our mismatched-write
    experiment (Section~\ref{sec:res-dm1}) shows that incorrect expert--centroid
    mapping produces $R \approx 0$ rather than partial degradation.
  \item \textbf{Continuous replay necessity}: Sleep deprivation studies
    report that blocking hippocampal replay degrades consolidation
    \citep{rasch2013sleep}. Similarly, our one-shot seeding experiment
    (Section~\ref{sec:res-dm2}) confirms that a single replay event is
    insufficient---continuous re-entry is required ($R_{\text{one-shot}}
    = 0.305$ vs.\ $R_{\text{continuous}} = 0.978$).
  \item \textbf{Neurogenesis--memory tradeoff}: Computational models of
    adult neurogenesis predict that replacement rate interacts with pattern
    separation capacity \citep{aimone2014regulation}. Our $(K, \pInject)$ regime
    map (Section~\ref{sec:res-envelope}) reveals the analogous
    tradeoff: more expert groups ($K$) require higher seeding rates ($\pInject$)
    to maintain persistent memory, with a characteristic interaction between
    expert count and injection rate at the degradation boundary.
\end{itemize}
These correspondences do not establish the analogy as mechanism. They are
functional; the order in which they were derived and measured was not
pre-registered, so they should not be read as tests of the analogy; the
underlying blocks are pilot-scale (4--5 seeds per condition;
Section~\ref{sec:protocol}); and the fast--slow division runs in the opposite
direction from CLS (Section~\ref{sec:dm}).

\subsection{Operating Bounds and Failure Modes}
\label{sec:bounds}

DM's failure modes are as informative as its successes. Three classes of
failure emerged:
\begin{itemize}
  \item \textbf{Mapping error} (Section~\ref{sec:res-dm1}): Incorrect
    expert--centroid mapping produces catastrophic failure ($R \approx 0$),
    not graceful degradation. This indicates that DM's efficacy depends
    critically on the fidelity of the recording--expert correspondence.
  \item \textbf{Insufficient injection} (Section~\ref{sec:res-envelope}):
    Low injection rate $\pInject$ leads to progressive degradation as the
    proportion of randomly initialized units increases, with higher expert
    counts requiring correspondingly higher injection rates.
  \item \textbf{Single-shot inadequacy} (Section~\ref{sec:res-dm2}):
    One-shot seeding fails to achieve persistent reconstruction
    (Section~\ref{sec:res-dm2}), consistent with the continuous re-entry
    requirement of the triple-loop architecture.
\end{itemize}

These failure modes are informative because they are not what the simplest
alternatives predict: neither generic initialization effects nor
population-level averaging accounts for a \emph{catastrophic} rather than
graded response to a mismatched mapping. Those two alternatives were tested
(Sections~\ref{sec:res-dm1} and~\ref{sec:res-dm2}); we did not enumerate the
space of possible explanations, so this is evidence against the accounts we
checked rather than a demonstration that no other account fits.

\subsection{Relation to Gradient-Based Memory Mechanisms}
\label{sec:related}

DM occupies a distinct position in the memory mechanism landscape. Unlike
EWC \citep{kirkpatrick2017ewc} and SI \citep{zenke2017si}, which protect
parameters via Fisher information or path integrals (both requiring gradient
access), DM operates through population-level centroid tracking and
re-injection. Unlike Titans \citep{titans2024}, which gates memory writes
via surprise computed through gradient descent, DM's recording is
unconditional---all active expert content contributes to the centroid
update.

DM shares the ``fast store / slow consolidation'' vocabulary of complementary
learning systems \citep{kumaran2016hippocampus} but assigns the roles in the
opposite direction, and without gradient computation: the recorded centroid is
the \emph{slow} term, a population-level moving average, while the continually
replaced active units are the fast one (Section~\ref{sec:dm}). Continuous
seeding is what carries the slow term back into the fast population. The Bayesian
Confidence Propagation Neural Network (BCPNN;
\citealt{sandberg2003bcpnn, lansner2009bcpnn}) provides Hebbian learning
without backpropagation, but its edge-level weights are destroyed by unit
turnover. DM complements Hebbian plasticity by providing a persistence
mechanism that survives the stochastic replacement events that define
dissipative architectures.

Beyond gradient-based methods, DM relates to several backpropagation-free memory
paradigms. Classical Hopfield networks \citep{hopfield1982neural} and
their modern continuous extensions \citep{ramsauer2021hopfield} store
patterns as fixed-point attractors in an energy landscape; however, they
assume fixed network topology without unit turnover.
Our empirical comparison (Section~\ref{sec:res-p1}) confirms this
theoretical limitation: under matched stochastic turnover,
Hopfield representation quality drops from $R = 0.999$
to $0.474$, while DM maintains $R = 0.854$.
Reservoir computing and echo state networks
\citep{jaeger2001esn, lukosevicius2009reservoir} maintain fading memory
through recurrent dynamics without gradient-based training, but their
memory timescale is bounded by the spectral radius and decays
exponentially. An ESN with supervised linear readout achieves higher raw
reconstruction ($R = 0.943$) but uses a fundamentally different learning
paradigm; notably, DM exceeds the ESN on functional separability
($S_{\text{ctx}} = 0.900$ vs.\ $0.857$), suggesting that expert-based
organization provides structural advantages beyond reconstruction accuracy.
Attractor-based models of working memory
\citep{wang2001synaptic, compte2000synaptic} sustain representations
through self-excitation in fixed populations; DM differs in that its
population is non-stationary (units are continuously replaced), requiring
an external consolidation loop rather than intrinsic attractor dynamics.
DM's expert routing also connects to the literature on competitive
learning \citep{grossberg1987competitive, kohonen1982selforganizing}
and modular brain organization \citep{sporns2016modular}, where
specialized subnetworks emerge through competitive resource allocation.
The selectionist perspective of neural Darwinism
\citep{edelman1987neural}---in which populations of neurons are selected
for by environmental demands---provides a conceptual framework for our
system's unit turnover and replacement dynamics.

\paragraph{Relation to spiking neural network memory mechanisms.}
The spiking neural network (SNN) literature has developed several
distinct strategies for endowing spike-based architectures with
long-range temporal memory, each of which warrants explicit comparison
with DM.
\emph{Long-Short-Term Spiking Neural Networks} (LSNN;
\citealt{bellec2018lsnn}) introduce adaptive-threshold LIF neurons
(ALIF) whose firing threshold decays on a slow timescale
($\tau_{\text{adapt}} \gg \tau_{\text{mem}}$), enabling the threshold
variable to act as an implicit memory trace across timescales of
seconds---the closest SNN analogue to LSTM state cells.
\emph{Liquid State Machines} (LSM; \citealt{maass1997lsm}) and their
modern variants exploit the rich recurrent dynamics of a fixed spiking
reservoir to maintain \emph{fading} memory: representations decay with a
timescale governed by the spectral radius, analogous to echo state
networks \citep{lukosevicius2009reservoir}.
\emph{Spikformer} \citep{zhu2023spikformer} replaces floating-point
self-attention with spike-AND operations, using the resulting attention
weight matrix as a form of global token-level memory across the sequence
length.
\emph{Memory-Augmented SNNs} (MA-SNN), which adapt the Neural Turing
Machine paradigm \citep{graves2014ntm} to spiking architectures, attach
an external associative memory matrix and use spike timing as an address
signal for read and write operations, enabling non-volatile episodic
storage beyond the recurrent horizon.

DM differs from all four paradigms along two complementary axes.
First, \emph{computational substrate}: DM operates on
continuous-valued content vectors updated by local Hebbian-type rules,
without spike encoding, surrogate-gradient training, or spike-AND logic.
Second, and most critically, \emph{problem definition}: ALIF and LSM
address how to extend the effective temporal horizon within a fixed-weight
spiking network; Spikformer leverages attention to aggregate sequence-wide
context; MA-SNN addresses where to buffer episodic records.
None of these frameworks confronts our central challenge---\emph{how to
maintain persistent representations when every learnable unit is
periodically destroyed and reborn with random state}.
Stochastic unit turnover eliminates the slow-decaying threshold of ALIF
(it resets on death), destroys the recurrent trajectory of LSM (the
reborn unit injects random dynamics), and disconnects Spikformer's query
vectors (no persistent state survives). MA-SNN's external memory can in
principle survive unit death, but its write gate is explicitly triggered by
spike events---a mechanism absent from backpropagation-free dissipative
architectures. DM's triple-loop (recording $\to$ seeding $\to$
stabilization) addresses the \emph{cross-turnover persistence problem}
without gradient computation and without spike encoding, and
Proposition~\ref{prop:drift} bounds the dispersion it leaves. Two
qualifications: we make no priority claim, having not surveyed the field
systematically enough to assert that no prior mechanism targets this problem;
and the centroid store is explicitly external to the unit population
(Section~\ref{sec:dm}), so this is not a mechanism that avoids separate
storage.

\subsection{Relation to Adjacent Traditions}
\label{sec:adjacent}

\paragraph{Cognitive architectures.} Several hundred cognitive architectures
have been proposed over four decades \citep{kotseruba2020architectures}, and
we do not position DM as another entry in that catalogue. The scope here is
narrower: one mechanism inside one architecture, addressed to one question ---
whether context-specific memory can survive continuous stochastic replacement
of the units that hold it. Where full architectures are assessed on breadth of
cognitive function, the claim tested here is a mechanism-level one, and the
appropriate comparison set is memory mechanisms rather than architectures.

\paragraph{What is recombined, and what is not.} It is worth separating the
components, which are established, from the arrangement, which is what we
propose. The recording loop is an exponential moving average, that is,
classical stochastic approximation \citep{robbins1951stochastic};
prototype- and centroid-based memory is long-standing; discrete gating is
standard in mixture-of-experts models \citep{jacobs1991adaptive,
shazeer2017moe}; and online sample-driven prototype allocation has an
extended history in evolving connectionist systems
\citep{kasabov2007evolving}. None of these components is introduced here.
What is specific to this work is the \emph{closure}: recording and seeding
are wired into a re-entrant loop whose function is to survive turnover of the
substrate. The empirical content is correspondingly narrow --- that neither
half of the loop suffices alone (Section~\ref{sec:res-ablation}), together
with the drift bound that makes the arrangement's effect explicit
(Proposition~\ref{prop:drift}). ``Triple-loop consolidation'' should therefore
be read as a name for that arrangement, not as a claim that any of the three
operations is itself new.

\paragraph{Dissipation: implementation property or theoretical commitment?}
The architecture is dissipative in a concrete and limited sense: units consume
a finite energy budget, and units that exhaust it are replaced. We use this as
a \emph{constraint the memory mechanism must survive}, not as a thermodynamic
account of cognition. The relation to dissipative-structure theory
\citep{prigogine1984order} is accordingly one of shared problem shape ---
maintaining organisation under continuous throughput --- rather than a
derivation from free-energy or entropy-production principles. We state this
explicitly because energy-based accounts of cognition would predict additional
structure, for instance an extremum principle governing the turnover rate,
that this work neither assumes nor tests.

\paragraph{Agency.} DM is a memory mechanism within an architecture. It is not
an architectural principle for autonomous agents, and nothing in these results
requires the system to select goals, act, or be embodied. We note this because
``cognitive architecture'' is sometimes read as implying agency, and the
experiments reported here do not speak to it.

\paragraph{Neuromorphic implementation: an implication, not a result.} The
properties this mechanism relies on --- local updates, no backward pass, and
tolerance of continuous element replacement --- are among those that
distinguish neuromorphic substrates from conventional accelerators
\citep{davies2021loihi2}, and a recording--seeding--stabilisation cycle is in
principle mappable onto hardware whose elements degrade or are reconfigured.
We state this as a direction for future work and nothing further. The evidence
here is a simulation of a mechanism: it establishes nothing about power, area,
latency, or hardware feasibility, and we report no such measurement. The
operating point of any such implementation would in particular depend on how
the number of distinguishable stored contents interacts with activity level,
which is exactly where a mechanism maintaining $K$ distinct centroids would
have to be characterised before an efficiency claim of any kind could be
entertained.

\subsection{Limitations}
\label{sec:limitations}

Several limitations warrant acknowledgment. First, the experiments
presented here use synthetic multi-context tasks to enable precise
control of ground-truth structure; systematic characterisation on
naturalistic domains remains ongoing work. Second, the expert count $K$ is
fixed; dynamic expert creation or destruction---which would model
developmental or adaptive changes in modular organization---is not
addressed. Third, DM centroids are population-level summaries;
per-unit episodic memory (specific experiences rather than
prototypical representations) requires additional mechanisms.
Fourth, the hippocampal consolidation analogy is functional, not
mechanistic: we do not claim that DM implements the biological
substrate of hippocampal memory, only that both systems solve a
shared functional problem through parallel strategies. The analogy is
limited in a further way that is worth stating plainly: this is a discrete,
synchronous, digital simulation, whereas the biological systems it draws on
are continuous, asynchronous, and situated in a body. Correspondence at the
level of computational function does not transfer to those levels, and we
claim none.

Fifth, the causal evidence presented here is interventional rather than
structural. Proposition~\ref{thm:t3} establishes centroid collapse
analytically under Uniform Activation; the empirical claims rest on
single-factor removal, a $2^3$ factorial design, and two sham controls
(mismatched mapping and noise seeding) rather than on a structural causal
model. What this design supports is that removing a component removes the
effect within the parameter ranges scanned. It does not establish necessity
outside those ranges, nor does it identify this mechanism uniquely among
alternatives we did not test.

\section{Conclusion}
\label{sec:conclusion}

We have introduced Deep Memory (DM), a triple-loop consolidation mechanism
that provides persistent, context-specific memory in a backpropagation-free
dissipative cognitive architecture. Through systematic experimentation
($972$ simulation runs across thirteen experimental blocks), we found that:
(1)~discrete expert routing is required for context-specific
memory; (2)~DM achieves high-fidelity persistent representations
($R = 0.984$) and near-complete reconstruction after interference
($R_{\text{recon}} = 0.978$); (3)~the mechanism operates within a
well-characterized $(K, \pInject)$ envelope with reproducible degradation boundaries;
(4)~DM is invariant to scheduling pattern; (5)~recording $\times$
seeding is the minimal sufficient combination; and (6)~under matched turnover
DM exceeds the Hopfield baseline on representation quality, while the
supervised-readout ESN attains a higher $R$ and DM exceeds it only on functional
separability.

These results demonstrate that backpropagation-free cognitive systems can achieve
persistent memory through architectural constraints---specifically, through
the interaction of discrete routing (which prevents representational
collapse) and triple-loop consolidation (which maintains population-level
attractors across stochastic turnover). At the functional level only, this
is consistent with one premise of complementary learning systems
theory---that a store kept apart from the computing substrate can stabilise
it---carrying over to a backpropagation-free dissipative system, with the
fast and slow roles reversed; it says nothing about biological mechanism.
Apart from the single-factor ablation, the results rest on 4--5 seeds per
condition and are effect-size estimates that still require confirmatory
replication at $\geq 14$--20 seeds per condition.


\section*{Data Availability Statement}
No external datasets were used.
All observations reported in this study were generated at run time
by a fully synthetic simulation under the experimental protocols
described in the manuscript and Supplementary Material.
Summary statistics necessary to support the main claims are reported
in the paper and Supplementary Material.

\section*{Code Availability Statement}
The simulator used to generate the reported results is proprietary and is
the subject of a pending patent application on which the author is the named
inventor; its source code therefore cannot be released publicly. So that each reported
result can nonetheless be specified and checked, we provide the following.
(i)~Supplementary Table~\ref{tab:s-repro} gives, for every figure and table,
the experimental block, conditions, number of runs, exact random seeds, and
all parameter values that differ from the simulator defaults.
(ii)~Supplementary Data Sheet~1 gives the value of every default parameter of
each simulator version used, under anonymised identifiers grouped by
functional category, and maps each block to its version.
(iii)~The complete simulator source, the per-block experiment scripts, and
the key from identifiers to parameter names are
available to the handling editor and reviewers on a confidential basis, on
request to the corresponding author; together with (i) and (ii), they
determine the full configuration of every run.
(iv)~A self-contained Python/NumPy reference implementation of the
triple-loop mechanism (\texttt{dm\_minimal\_reproduction.py}, ${\sim}200$
lines) is provided as Supplementary Material. It is not the simulator and
uses illustrative parameter values; it reproduces the qualitative DM versus
no-DM contrast under matched stochastic turnover ($R_{\text{DM}} = 0.984$
vs.\ $R_{\text{No-DM}} = 0.457$, 5 seeds), not the specific values reported
in the paper.

\section*{Author Contributions}
JL: Conceptualization, Methodology, Software, Validation,
Formal Analysis, Investigation, Data Curation, Writing---Original Draft,
Writing---Review \& Editing, Visualization, Project Administration.

\section*{Funding}
This research received no external funding.

\section*{Conflict of Interest}
JL is the founder of RailMind Systems, a company that develops commercial
technology related to the architecture studied in this paper, and is the named
inventor on a pending patent application that covers aspects of the
memory-seeding mechanism described here. The simulator used in this study is
proprietary for this reason (see the Code Availability Statement). The author
declares no other commercial or financial relationships that could be
construed as a potential conflict of interest.

\section*{Acknowledgments}
AI-assisted tools (Claude, Anthropic) were used during the preparation of
this manuscript for code development, data analysis scripting, and
drafting assistance. All scientific content, experimental design, analysis
decisions, and conclusions were made by the authors.

\bibliographystyle{plainnat}
\bibliography{refs}

\clearpage
\appendix
\renewcommand{\thesection}{S\arabic{section}}
\renewcommand{\thetable}{S\arabic{table}}
\renewcommand{\thefigure}{S\arabic{figure}}
\setcounter{section}{0}
\setcounter{table}{0}
\setcounter{figure}{0}

\section*{Supplementary Material}

\section{Single-Factor Causal Ablation (E1)}
\label{sup:e1}

Table~\ref{tab:s1-e1} reports the complete per-condition statistics for the
single-factor causal ablation ($K = 3$, 14 seeds per condition, $n = 112$
runs). Each condition changes one element relative to Full DM: a DM
component (\textsc{No\_DM}, \textsc{Seed\_off}, \textsc{DM\_wrong}), the
expert assignment (\textsc{No\_MoE}), unit turnover (\textsc{No\_turnover}),
or threshold adaptation (\textsc{Threshold frozen}); \textsc{Baseline}
removes both DM and context-bound assignment. $R$: representation fidelity,
mean over seeds. $S_{\text{ctx}}$ was not computed in this experiment.

\begin{table}[h]
\centering
\caption{E1 single-factor ablation: per-condition summary statistics.}
\label{tab:s1-e1}
\small
\begin{tabular}{@{}lcl@{}}
\toprule
Condition & $R$ (mean) & Key observation \\
\midrule
Full DM           & 0.983 & Reference (all components active) \\
No DM             & 0.296 & No memory mechanism \\
Seed off          & 0.296 & Identical to No DM in every seed \\
DM wrong expert   & 0.557 & Intermediate (content-specific) \\
No MoE            & 0.804 & Context-independent expert cycling \\
No turnover       & 0.997 & Turnover is the source of degradation \\
Threshold frozen  & 0.983 & Threshold adaptation not required \\
Baseline          & 0.317 & No DM and no context-bound assignment \\
\bottomrule
\end{tabular}
\end{table}

The identity between \textsc{Seed\_off} and \textsc{No\_DM} (all
per-seed differences are zero) confirms that seeding is the sole delivery
channel: without it, recorded centroids are never deployed to replacement
units, and anchoring alone cannot maintain representations against
dissipative drift.

\section{Mediation Analysis and K-Scaling (E2--E3)}
\label{sup:e2e3}

\paragraph{E2: Dose-response mediation ($n = 60$).}
Five injection doses (0\%, 10\%, 25\%, 50\%, 100\% of nominal
$p$) plus a noise-injection control are tested at $K = 3$ with 10 seeds
each ($6 \times 10 = 60$ runs). In every seed the rank correlation between
injection dose and $R$ is $\rho = 1.0$. An ordinary least-squares regression
on ranks of $R$ on $S_{\text{ctx}}$, with injection dose and seed as
covariates, gives $\beta = 0.909$ ($t = 14.9$, $R^2 = 0.982$). The noise-injection
condition produces metrics indistinguishable from the zero-dose baseline,
confirming that content-specific information---not injection
mechanics---drives the effect.

\paragraph{E3: K-scaling ($n = 54$).}
The DM advantage is tested at $K = 5$ and $K = 8$ (in addition to the
$K = 3$ used throughout the main experiments), with 6 seeds per
condition. At both scales, DM outperforms the no-DM baseline in every seed
(one-sided sign test $p = 0.016$).
$S_{\text{ctx}}$ at full injection dose: $0.969$ ($K = 5$) and $0.959$ ($K = 8$). The slight
decrease with $K$ is consistent with the $(K, \pInject)$ envelope
(Section~\ref{sec:res-envelope}), where larger $K$ requires higher $p$
for equivalent quality.

\section{Turnover-Dose Interaction (E4)}
\label{sup:e4}

E4 crosses two turnover levels---normal, and low (energy reset raised
fourfold, which lengthens mean unit lifespan by a factor of $4.4$)---with
three injection levels ($p \in \{0, 0.25, 1\}$ of nominal; $K = 8$, 6 seeds,
$2 \times 3 \times 6 = 36$ runs).

The DM effect on $R$ is essentially the same at both turnover levels
($\Delta R = 0.753$ normal, $0.763$ low), indicating that DM
maintains a fixed quality advantage regardless of disruption frequency.
In contrast, the DM effect on $S_{\text{ctx}}$ shows a dose interaction:
$\Delta S_{\text{ctx}}$ ranges from $0.420$ (low turnover) to $0.785$
(normal turnover), suggesting that DM's separability contribution is
most needed---and most effective---at the architecture's native
operating point. Under low turnover, units survive long enough for
Hebbian learning alone to achieve moderate separability; under normal
turnover, only DM-seeded units maintain expert-specific content.

Verdict: Partial (4/5 criteria met). The constant-$\Delta R$ finding
is a strong positive result; the interaction on $S_{\text{ctx}}$
provides mechanistic insight into dual channels of DM action.

\section{Soft vs.\ Hard Routing (E5)}
\label{sup:e5}

E5 replaces the hard MoE mask by a leaky mask in which non-selected
units receive a fraction $\alpha$ of the input ($\alpha \in \{0, 0.10, 0.30,
1\}$; $\alpha = 0$ is hard routing and $\alpha = 1$ removes the mask), with
and without DM ($K = 8$, 6 seeds per condition, $n = 48$).

With DM, hard routing achieves $S_{\text{ctx}} = 0.959$; a small leak
leaves this essentially unchanged ($0.959$ at $\alpha = 0.10$, $0.956$ at
$\alpha = 0.30$), whereas removing the mask ($\alpha = 1$) reduces it to
$0.812$. The degradation without a mask follows a dual pathway: (i)~cross-expert content
interference (units receive input from multiple contexts, diluting
expert-specific centroids) and (ii)~amplified unit turnover
(energy distribution across experts increases effective replacement rate by
$4.4\times$).

The DM fidelity advantage ($\Delta R$) is larger under hard routing
than with no mask ($0.753$ vs.\ $0.697$), consistent with DM and hard routing
being synergistic: DM maintains centroids that hard routing keeps
functionally separated.

Verdict: Partial (3/4 criteria met). An unexpected mild regularization
effect at $\alpha = 0.10$ warrants further investigation but does not
affect the primary conclusion.

\section{Associative and Reservoir Baselines: Implementation and Analysis (E6)}
\label{sup:e6}

\paragraph{Modern Hopfield network.}
We implement a modern (continuous) Hopfield network
\citep{ramsauer2021hopfield} with $M = 32$ memory slots, inverse
temperature $\beta_H = 8.0$, and Hebbian storage via EMA update on the
closest slot (learning rate $\eta_H = 0.01$). Slots are initialised at the
context prototypes plus small Gaussian noise (s.d.\ $0.01$), which accounts
for the near-perfect $R$ without turnover. Stochastic turnover resets
each slot independently with probability $\turnover$ per step, replacing the
stored pattern with a random vector.

\paragraph{Echo state network.}
We implement a standard ESN \citep{jaeger2001esn} with reservoir size
$N_{\text{res}} = 256$, spectral radius $\rho_W = 0.9$, and
LMS/Widrow-Hoff linear readout (learning rate $\eta_E = 0.001/N_{\text{res}}$).
Turnover resets each reservoir neuron independently with probability
$\turnover$, clearing both the neuron's state and its corresponding readout
weight column---analogous to the stochastic turnover in our architecture,
which resets both unit state and the associated edge-level weights.

\paragraph{Turnover calibration.}
The turnover rate ($\turnover = 0.0132$ per unit per step) is calibrated to
match the empirical death rate of the dissipative grid under standard
operating conditions (hard routing, no injection, $K = 8$).

\paragraph{ESN regularization effect.}
ESN with turnover ($R = 0.943$) outperforms ESN without turnover
($R = 0.831$). This counterintuitive result arises because readout weight
reset acts as implicit regularization (analogous to dropout): the LMS
readout rapidly relearns from the intact reservoir, while the periodic
reset prevents weight accumulation artifacts. This effect is specific to
supervised readout systems and does not transfer to the unsupervised
DM setting.

\paragraph{Reproduction manifest.}
Table~\ref{tab:s-repro} lists, for every figure and table, the experimental
block and its conditions, the number of runs, and the exact random seeds.
Every run sets the simulator seed to the listed value; where a Deep Memory
store is initialised it uses seed $+1000$, and the context schedule in E1 uses
seed $+5000$. Unless stated otherwise in the table, all blocks share one
operating point that overrides the simulator defaults as follows: firing
threshold $a_{\text{fire}} = 0.3$, turnover pressure $\mu_{\text{exc}} = 20$,
Hebbian rate $\lambda = 0.01$, hard MoE routing with balanced, fixed
context--expert assignment and per-expert thresholds, context blocks of 250
steps, DM recording rate $\lrRecord = 0.01$, and seeding noise
$\sigma = 0.05$. All other parameters take the simulator defaults listed in
Supplementary Data Sheet~1, which gives every default value for each of the
simulator versions used, under anonymised identifiers grouped by functional
category, and maps each block to its version.

\begin{table}[h]
\centering
\caption{Reproduction manifest: conditions, runs and random seeds for every
  figure and table. $K$: number of contexts and experts.}
\label{tab:s-repro}
\small
\begin{tabular}{@{}llp{5.6cm}rl@{}}
\toprule
Item & Block & Conditions (non-default settings) & Runs & Seeds \\
\midrule
Fig.~1 & --- & Schematic; no simulation & --- & --- \\
Fig.~2 & G1 & 8 routing arms, $K=3$ & 56 & 0--6 \\
Fig.~2 & Sham & 5 arms incl.\ per-cycle and per-block permutation, $K=3$ & 35 & 0--6 \\
Fig.~3 & DM1 & 4 recording arms, $K=3$ & 16 & 0--3 \\
Fig.~4 & DM2 & 6 seeding conditions (noise seeding $\sigma=0.5$), $K=3$ & 30 & 0--4 \\
Fig.~5 & DM3 & 10 $(K,p)$ cells, 23 conditions incl.\ shared controls & 115 & 0--4 \\
Fig.~6A & G2B & 5 schedulers $\times$ 3 arms, $K=3$, $p=0.25$ & 75 & 42, 123, 456, 789, 2024 \\
Fig.~6B & DM3 map & $K \in \{3,5,8\}$ $\times$ 2 schedulers $\times$ $p \in \{0.1,0.25,0.5\}$ $\times$ 2 arms, plus 55 stress runs & 235 & 42, 123, 456, 789, 2024 \\
Fig.~7 & DM-F & $2^3$ factorial (recording, seeding, anchoring), $K=3$ & 40 & 0--4 \\
Table~1; Suppl.~S1 & E1 & 8 single-factor conditions, $K=3$ & 112 & 0--13 \\
Table~1; Suppl.~S2 & E2 & 5 injection doses $+$ noise control, $K=3$ & 60 & 0--9 \\
Table~1; Suppl.~S2 & E3 & 4 ($K=5$) $+$ 5 ($K=8$) conditions & 54 & 0--5 \\
Table~1; Suppl.~S3 & E4 & 2 energy-reset levels (100, 400) $\times$ 3 injection levels, $K=8$ & 36 & 0--5 \\
Table~1; Suppl.~S4 & E5 & mask leak $\alpha \in \{0,0.1,0.3,1\}$ $\times$ 2 injection levels, $K=8$ & 48 & 0--5 \\
Table~1; Suppl.~S5 & E6 & Full DM, No DM; Hopfield and ESN, each with and without turnover; $K=8$ & 60 & 0--9 \\
\bottomrule
\end{tabular}
\end{table}

\paragraph{Experiment registry.}
Table~\ref{tab:s-registry} provides a complete registry of all
experimental blocks reported in this paper.

\begin{table}[h]
\centering
\caption{Complete experiment registry ($972$ runs across 13 blocks).}
\label{tab:s-registry}
\small
\begin{tabular}{@{}clccc@{}}
\toprule
Block & Purpose & $n$ & Primary verdict & Section \\
\midrule
G1    & MoE causality                   & 56  & Confirmed & \ref{sec:res-causality} \\
DM1   & Persistent representations      & 16  & Confirmed & \ref{sec:res-dm1} \\
DM2   & Functional reconstruction       & 30  & Confirmed & \ref{sec:res-dm2} \\
DM3   & $(K,p)$ envelope + regime map & 350 & Confirmed & \ref{sec:res-envelope} \\
G2B   & Scheduling invariance           & 75  & Confirmed & \ref{sec:res-scheduling} \\
DM-F  & Factorial ablation              & 40  & Confirmed & \ref{sec:res-ablation} \\
Sham  & Binding disruption control      & 35  & Confirmed & \ref{sec:res-causality} \\
E1    & Single-factor ablation          & 112 & Confirmed & \ref{sec:res-p1} \\
E2    & Mediation analysis              & 60  & Confirmed & \ref{sec:res-p1} \\
E3    & $K$-scaling                     & 54  & Confirmed & \ref{sec:res-p1} \\
E4    & Turnover-dose interaction       & 36  & Partial   & \ref{sup:e4} \\
E5    & Soft vs.\ hard routing          & 48  & Partial   & \ref{sup:e5} \\
E6    & Associative/reservoir baselines          & 60  & Confirmed & \ref{sec:res-p1} \\
\midrule
      & \textbf{Total}                  & \textbf{972} &  & \\
\bottomrule
\end{tabular}
\end{table}

A run is defined as one completed simulation instance under a unique
(configuration, seed) pair that produces the full set of logged metrics.
The 91 runs cited for the causality test in Section~\ref{sec:res-causality}
are the G1 and Sham blocks together; each run is counted once in the total.
The DM3 block combines the initial 10-point envelope sweep ($n = 115$;
Section~\ref{sec:res-envelope}) with the extended $K \times p$ regime
map ($n = 235$; Section~\ref{sec:res-scheduling}).

\end{document}